\documentclass[11pt]{article}

\usepackage{fullpage}
\usepackage{amsmath}
\usepackage{amsfonts}
\usepackage{amsthm}
\usepackage{hhh}
\usepackage{subfigure}
\usepackage{kpfonts}

\usepackage{tablefootnote}

\usepackage[colorlinks,
            linkcolor=red,
            anchorcolor=blue,
            citecolor=blue
            ]{hyperref}
\usepackage{enumitem}
\usepackage{mathtools}

\usepackage{colortbl}
\definecolor{LightCyan}{rgb}{0.8, 0.9, 1}

\usepackage{amssymb}
\usepackage{pifont}
\newtheorem{Condition}[theorem]{Condition}
\newtheorem{Desideratum}[theorem]{Desideratum}

\ifdefined\final
\usepackage[disable]{todonotes}
\else
\usepackage[textsize=tiny]{todonotes}
\fi
\allowdisplaybreaks

\title{\huge Towards a Principled Muon under $\mup$: Ensuring Spectral Conditions throughout Training}

\author{John Zhao}
\date{}

\newcommand{\mup}{\mu\mathsf{P}}

\newcommand{\algname}{Muon++}

\begin{document}

\maketitle

\begin{abstract}
     The $\mu$-parameterization ($\mup$) provides a principled foundation for large language model (LLM) training by prescribing width-independent learning dynamics, which in turn enables predictable scaling behavior and robust hyperparameter transfer across model sizes. A central requirement of $\mup$ is the satisfaction of certain spectral conditions on weight matrices, which ensure consistent feature learning and optimization behavior as model width grows. While these conditions are well understood in theory, guaranteeing their validity in practical training for matrix-based optimizers such as Muon is still under studied. Existing works that study Muon under $\mup$ exhibit important limitations: they either do not ensure that the spectral conditions hold throughout the entire training horizon, or require repeated spectral normalization (or Newton–Schulz iterations) applied to both weights and updates, leading to significant computational overhead and reduced practicality. In this work, we show how to reliably guarantee the spectral conditions required by $\mup$ for Muon during the entire training process. Our key insight is that for moderately large models, maintaining spectral control at the level of optimizer updates alone is sufficient to preserve $\mup$-compatible scaling, eliminating the need for explicit spectral normalization of the weights. Based on this principle, we develop a variant of Muon, namely Muon++, that satisfies spectral condition throughout the training process. Our results bridge the gap between the theoretical promises of $\mup$ and the practical deployment of matrix-based optimizers in long-horizon training. We also take the first step towards an adaptive spectral condition by incorporating data-dependent effects, making it better suited for long-horizon LLM training.
\end{abstract}

\section{Introduction}
Recently a lot of works have demonstrated the potential of scaling up large language models (LLMs) with the guidance of scaling laws~\citep{zhao2023survey,filatov2025optimal,fan2025robust}. Traditional empirical scaling laws usually model the optimal hyper-parameters, such as learning rate, compute budget and model size, through extrapolation from considerable numbers of grid search experiments~\citep{kaplan2020scaling,hoffmann2022training,krajewski2024scaling,bi2024deepseek,li2025predictable}. Such an empirical approach suffers from high resource consumption and high sensitivity to model structures.

To model hyper-parameter transfer and ensure feature learning,~\citet{yang2021tensor} proposed Maximal-update Parameterization ($\mup$), which analyzes the changes and invariants within dynamics of model parameters. In detail, under certain conditions on the weight initialization and learning rate multiplier, the optimal learning rate is transferrable with given optimizers across large widths~\citep{yang2022tensorv} and depths~\citep{TensorProgramVI}. In addition,~\citet{yang2023spectral} proposed that, it is sufficient to attain such transferability with the condition for spectral scaling when the spectral norms of the weight matrices with shape $(n_{\ell}, n_{\ell-1})$ and their gradient updates are both $\Theta(\sqrt{{n_\ell}  / {n_{\ell-1}}})$. 

However, turning spectral conditions into practical guarantees during large-scale training, particularly for matrix-based optimizers that directly manipulate the geometry and spectrum of parameter updates, remains under studied. This gap is especially salient for Muon~\citep{muon}, a recently proposed matrix-level optimizer that has been applied in the training of a lot of LLMs~\citep{team2025kimi,Zai2025GLM45}. Despite its promise as a structured alternative to elementwise adaptive methods such as AdamW~\citep{Adamw}, it is unclear under what algorithmic choices Muon can be made genuinely compatible with $\mup$ in realistic training regimes. Several initial attempts~\citep{yang2023spectral,bernstein2025deriving,shah2025practical,kexuefm-11416,su2025high,Zhihu2025} have explored the design of Muon under $\mup$. However, these approaches exhibit important limitations. On one hand, some approaches~\citep{bernstein2025deriving,kexuefm-11416} only enforce the $\mup$-required spectral control in a limited sense (e.g., near-initialization or for a single update), without guaranteeing that the spectral conditions persist throughout the entire training horizon. On the other hand, approaches that aim to maintain these conditions over time~\citep{yang2023spectral,su2025high,Zhihu2025} often resort to repeated spectral normalization applied to both the weights and the updates, which introduces substantial computational overhead and undermines practicality at LLM scale. 

The aforementioned limitations motivate our work. Our goal is to develop a principled and practical formulation of Muon under $\mup$ that reliably maintains the required spectral conditions throughout the entire optimization process, without incurring prohibitive computational overhead.
Our key insight is that for moderately large models, enforcing spectral control at the level of optimizer updates alone is sufficient to preserve $\mup$-compatible scaling behavior. This eliminates the need for explicit spectral normalization applied to the weights themselves, which are a major source of inefficiency in prior approaches.
Building on this insight, we introduce a variant of Muon, namely \algname{}, a modified Muon optimizer that provably satisfies the $\mup$-required spectral conditions throughout training. Our approach bridges the gap between the theoretical guarantees of $\mup$ and the practical deployment of matrix-based optimizers in LLM training.

Our contributions are summarized as follows:
\begin{itemize}
\item We show that maintaining spectral control at the level of optimizer updates is sufficient for preserving $\mup$-compatible scaling, thereby avoiding explicit spectral normalization of model weights.
\item Based on these principles, we propose \algname{}, a practical modification of the Muon optimizer that achieves stable and transferable training dynamics with affordable computational cost.
\item We also discuss how to reliably guarantee the spectral conditions required by $\mup$ for Muon across the entire training trajectory, even when the model width is very large.
\item Finally, we introduce an adaptive spectral condition that incorporates data-dependent effects, making Muon++ better suited for long-horizon LLM training.
\end{itemize}


\paragraph{Notation} Bold lowercase letters are reserved for vectors. Matrices are denoted by bold capital letters $\Wb, \Xb$, etc. $\dotp{\Ab}{\Bb} \coloneqq \tr(\Ab^\top \Bb)$ for any $\Ab, \Bb$ with compatible dimensions. For a matrix $\Mb $, $\|\Mb\|=\max_{\xb\ne\mathbf{0}}{\|\Mb\xb\|_2} / {\|\xb\|_2}$ denotes its spectral norm, $\|\Mb\|_F=\sqrt{\dotp{\Mb}{\Mb}}$ denotes the Frobenius norm, and $\|\Mb\|_*=\tr(\sqrt{\Mb^\top\Mb})=\max_{\|\Wb\|\le 1}\langle \Mb,\Wb\rangle$ is the nuclear norm. We use $a \wedge b$ (resp. $a \vee b$) to denote the minimum (resp. maximum) of $a$ and $b$. We use standard asymptotic notations including $o(\cdot), O(\cdot), \Omega(\cdot), \Theta(\cdot)$; and write $a_n \ll b_n$ (resp. $a_n \gg b_n$) for $a_n = o(b_n)$ (resp. $b_n = o(a_n)$), $a_n \lesssim b_n$ (resp. $a_n \gtrsim b_n$) for $a_n = O(b_n)$ (resp. $a_n = \Omega(b_n)$), and $a_n \asymp b_n$ for $a_n = \Theta(b_n)$. We put $X_n \stackrel{\mathrm{a.s.}}{\longrightarrow} X$ for random variables $X$ and $\{X_n\}_{n\geq1}$ if $\PP(X_n \to X) = 1$.

\section{Related Work}

\noindent\textbf{Matrix-based optimization methods.} To bypass the prohibitive computational cost of the Hessian computation in Newton methods, previous work largely focused on approximate preconditioners, such as diagonal~\citep{adagrad,adam,Adamw} and sketched~\citep{erdogdu2015convergence,gonen2015faster} approximations. The Shampoo optimizer~\citep{gupta2018shampoo} marked a significant advance by practically leveraging matrix-level information, which inspires several follow-ups, including Muon~\citep{jordan2024muon}. Muon is proposed by applying a Newton-Schulz iteration to the momentum. Subsequently~\citet{liu2025muon} improved Muon by incorporating weight decay and aligning the update magnitudes of AdamW and Muon to achieve a better performance. Following \citet{liu2025muon}, many practical matrix-level methods have been proposed.~\citet{liu2025mars} introduced variance reduction techniques to matrix-level optimizers. PolarGrad~\citep{lau2025polargrad} advances Muon using a preconditioned framework based on the polar decomposition of momentum. Scion~\citep{pethick2025training} and Gluon~\citep{riabinin2025gluon} generalize Muon under the linear minimization oracle (LMO) framework. In addition to the development of matrix-level optimizers, there is a line of works on the parameterization and hyper-parameter transfer of such approaches.~\citet{ishikawaparameterization} introduced the method for hyper-parameter transfer of methods like K-FAC~\citep{martens2015optimizing} and Shampoo.~\citet{shah2025practical} studied hyper-parameter transfer of Muon optimizer empirically. And~\citet{filatov2025optimal} discussed the optimal hyper-parameter transfer of Scion in terms of spectral condition in both model size and data scaling directions.

\noindent\textbf{Hyperparameter transfer.} Extensive research has focused on accelerating hyperparameter search for neural networks~\citep{snoek2012practical, snoek2015scalable, jamieson2016non, akiba2019optuna} and exploring methods trying to transfer optimal hyperparameters across different tasks or datasets \citep{horvath2021hyperparameter, perrone2018scalable,yogatama2014efficient}. A significant advancement in this area is the Maximal Update Parametrization ($\mup$), proposed by~\citet{TensorProgramVI}, which builds upon Standard Parametrization (SP) techniques like Xavier~\citep{glorot2010understanding} and Kaiming initialization~\citep{he2015delving}. $\mup$ framework unifies previous methods, including SP, Neural Tangent Kernel (NTK) parametrization~\citep{jacot2018neural}, and Mean Field parametrization~\citep{chizat2018global, mei2018mean, sirignano2020mean, rotskoff2022trainability}, while enabling feature learning that generalizes to infinite-width conditions. This foundation led to $\mu$\textit{Transfer}~\citep{yang2022tensor}, which enables zero-shot hyperparameter transfer across models with different hidden sizes. This transferability was generalized across different architectures and optimizers~\citep{yang2023tensor}, such as SGD, Adagrad~\citep{duchi2011adaptive}, and Adam~\citep{adam}, and the theory was extended to Depth-$\mup$ \citep{TensorProgramVI} and reformularized from a spectral norm perspective~\citep{yang2023spectral}. Recently, researchers have worked to generalize $\mup$~\citep{blakeu2025, meta2024llama4, haas2024effective, dey2024sparse, hajjar2024training} and apply its principles to other optimizers~\citep{ishikawaparameterization, everett2024scaling}. Notably, for the Muon optimizer, \citet{shah2025practical} attempted to directly apply the $\mu$Transfer theory developed for AdamW~\citep{Adamw}. And some previous works attempt to examine the spectral conditions of Muon under $\mup$, including~\citet{kexuefm-11416} and~\citet{pethick2025training}. However, these approaches can not satisfy the spectral conditions of $\mup$ simultaneously.

\section{Preliminaries}

\subsection{Maximal update parametrization ($\mup$)}
 
$\mup$~\citep{yang2021tensor} refers to scaling paradigms for weight initialization and optimizer configurations such that for any \emph{fixed} training step $t$ and \emph{fixed} model depth $L$, as the model width $n \to \infty$, $\forall \ell \in [L]$,
\begin{align}\label{eq:mup-feature-speed-formula}
    \hb_{\ell,t} = \Theta_n(1), \hb_{\ell,t+1} - \hb_{\ell,t} = \Theta_n(1), \cL_{t+1}(\hb_{L, t+1}) - \cL_{t}(\hb_{L, t}) = \Theta_n(1);
\end{align}
where $\hb_{\ell, t} \in \RR^{n} = \mathrm{activation}(\Wb^{\ell}_{ t}\hb_{\ell-1, t})$ is the output of the $\ell$-th layer at step $t$ and $\cL$ is the loss.\footnote{We consider models that are feed-forward at any \emph{local} level in this work.} The concrete scheme of $\mup$ vary for different optimizers~\citep{chizat2024feature,chen2025global,bernstein2025deriving}. We consider the scheme for Muon~\citep{bernstein2025deriving} in this work, whose original form is often derived from the so-called spectral conditions.

\subsubsection{Spectral conditions}

Since rigorous analyses under $\mup$ by far usually resort to mathematically involved random matrix machinery, and are restricted to only a constant number of update steps~\citep{yang2021tensor,hajjar2024training,chizat2024infinite}; it has been attractive for practitioners to recover the scaling scheme of $\mup$ via CLT- and LLN-type heuristics~\citep{yang2023spectral}, among which the \emph{spectral conditions}~\citep{yang2023spectral} has been widely deemed as a necessary condition for $\mup$ (consequently, for hyperparameter transfer).

\begin{Desideratum}[Feature learning]~\label{con:feature-learning}
Let $\hb_\ell(\xb)\in\RR^{n_\ell}$ denote the features of input $\xb$ at layer $\ell$ of a neural network, and let $\Delta \hb_\ell(\xb)\in\RR^{n_\ell}$ denote their change after a gradient step. We desire that:
\begin{align*}
    \|\hb_\ell\|_2=\Theta(\sqrt{n_\ell})\text{  and   }\|\Delta\hb_\ell\|_2=\Theta(\sqrt{n_\ell}),\text{   at layers }\ell=1,\cdots,L-1.
\end{align*}
\end{Desideratum}

\begin{Condition}[Spectral scaling]~\label{con:spectral-scaling}
Consider applying update $\Delta\Wb^\ell\in\mathbf{R}^{n_{\ell}\times n_{\ell-1}}$ to the weight matrix $\Wb^\ell\in\mathbf{R}^{n_{\ell}\times n_{\ell-1}}$ in $\ell$-th layer. The spectral norms of $\Wb^\ell$ and $\Delta\Wb^\ell$ should satisfy:
\begin{align*}
    \|\Wb^\ell\|_*=\Theta \left( \sqrt{\frac{n_\ell}{n_{\ell-1}}} \right),~\|\Delta\Wb^\ell\|_*=\Theta \left(\sqrt{\frac{n_\ell}{n_{\ell-1}}} \right),~\textnormal{at layers }\ell=1,\ldots,L.
\end{align*}
\end{Condition}

\Cref{con:spectral-scaling} implemented with $\mup$ only guarantees the desired spectral norms of the weight $\Wb$ and update $\bDelta$ across different widths, but \textbf{not} necessarily guarantees the desired spectral norms \emph{across entire time horizon of the training steps}. In other words, because $\mup$ is \textbf{data-independent}, the data-dependent nature of training intuitively may necessitate hard constraints for the training stability over time.

In fact, \citet{yang2023spectral} outlined another plausible way of implementing \Cref{con:spectral-scaling} for the standard parametrization (SP). 
In particular, \citet[Equations (16) and (17)]{yang2023spectral} implement it via explicit spectral normalization as\footnote{We ignore the momentum here and use the raw gradient instead only to avoid notation clutter.}
\begin{align}
    \bDelta_{t} &\leftarrow \frac{\nabla_{\Wb_{t}} \cL_{t}}{\norm{\nabla_{\Wb_{t}} \cL_{t}}}, \quad \Wb_{t+\frac{1}{2}} \leftarrow \Wb_{t} - \eta \sqrt{\frac{n_\ell}{n_{\ell-1}}}  \bDelta_{t}, \quad \Wb_{t+1} \leftarrow \sigma \sqrt{\frac{n_\ell}{n_{\ell-1}}} \frac{\Wb_{t+\frac{1}{2}}}{\norm{\Wb_{t+\frac{1}{2}}}}. \label{eq:cascade}
\end{align}
Such a two-step spectral normalization approach for SP has a trick issue: as long as $\norm{\Wb_{t+\frac{1}{2}}} \neq \sqrt{n_\ell / n_{\ell-1}}$, the normalization of weight will implicitly shrink or amplify the spectral magnitude of the previously normalized $\bDelta_{t}$ because $\frac{\Wb_{t+\frac{1}{2}}}{\norm{\Wb_{t+\frac{1}{2}}}} = \frac{\Wb_{t} - \eta \sqrt{\frac{n_\ell}{n_{\ell-1}}}  \bDelta_{t}}{\norm{\Wb_{t+\frac{1}{2}}}}$. Therefore, SP equipped with \Cref{eq:cascade} may not perfectly satisfy the spectral conditions for both the weight $\Wb_{t+1}$ and the one-step neat update $\bDelta_{t+1}$ simultaneously. Similar issues also exist for the spectral normalization algorithms discussed in \citep{su2025high,Zhihu2025}.
 
\subsection{Muon}
The Muon optimizer \citep{muon} was proposed to leverage the 2D geometric structure of model parameter matrices during training. Given a momentum coefficient $\mu \in (0,1)$ and a learning rate $\eta_t > 0$, the update rule is defined as:
\begin{align*}\Mb_t &= \mu \Mb_t + \nabla f(\Xb_{t},\bxi_{t}),\\
\mathbf{O}_t &= \text{NewtonSchulz}\left(\Mb_t\right),\\
\Xb_{t+1} &= \Xb_t-\eta_t\mathbf{O}_t.
\end{align*}
Here, $\Xb_t\in \RR^{m\times n}$ is the parameter matrix, and $\Mb_t$ represents the momentum matrix. The core of the update relies on the Newton-Schulz iteration \citep{bernstein2024old}, which is employed to approximate the polar factor $\Ub_t \Vb_t$, where $\Ub_t \bSigma_t \Vb_t$ is the Singular Value Decomposition (SVD) of $\Mb_t$. Moreover, instead of the original momentum $\Mb_t$, a Nesterov-style term $\mu\Mb_t+\nabla f(\Xb_{t},\bxi_{t})$ is usually used in Newton-Schulz iteration in practice. It is worth noting that Muon is designed only for matrix-shaped parameters. Vector-like parameters, including embeddings, normalization layers, and language model heads, are typically optimized by AdamW. 

However, empirical studies \citep{yuan2024mars,semenov2025benchmarking} have reported that the original Muon formulation yields unsatisfactory performance in large-scale language modeling tasks.~\citet{liu2025muon} identified the source of this suboptimal performance as a mismatch between the update magnitudes applied to the different parameter types. To address this disparity, they proposed a variant that rescales the Muon update to match this magnitude with a factor of $0.2\sqrt{\max(m,n)}$ for the weight $\Xb_t \in \RR^{m \times n}$. And this corrected version has demonstrated strong empirical performance on several benchmarks \citep{semenov2025benchmarking,wen2025fantastic}.

\section{The Proposed Algorithm: \algname{}}

\subsection{Formulation: Spectral conditions with spectral constraints}

Suppose $\Wb \in \RR^{n_{\ell} \times n_{\ell-1}}$ for the current step with a spectral norm in compliance with its spectral condition, i.e.,
\begin{align*}
    \norm{\Wb} = S \textcolor{red}{\coloneqq}  \sqrt{n_\ell / n_{\ell-1}},
\end{align*}
which can be achieved either by $\mup$ initialization (with high probability) or explicit normalization. Then our formulation is
\begin{equation}\label{eq:P} 
\begin{aligned}
    \max_{\bDelta: \norm{\bDelta} \leq 1} \dotp{\Gb}{\bDelta} \\
    \text{ subject to } \norm{ \Wb - \eta \cdot S\cdot\bDelta } = S,
\end{aligned}
\end{equation}
where $\Gb$ is usually replaced by a momentum estimator and $\eta>0$ is a constant step size that does not scale with the model size. 

\subsection{A projection-based solution without spectral normalization}
For the optimization problem in~\eqref{eq:P}, suppose the top singular value $S \coloneqq \sigma_1(\Wb)$ is unique, with the corresponding left and right singular vectors $\ub_1,\vb_1$; which holds almost surely for typical random weights with finite $n_{\ell} \vee n_{\ell-1}$. When the step size $\eta$ satisfies $\eta \cdot S \leq S - \sigma_2(\Wb)$, $\bDelta$ \emph{does not disturb} the top singular direction of $\Wb$, the constraint $\norm{ \Wb - \eta\cdot S\cdot\bDelta }=S$ is satisfied automatically; which is rigorously justified in \Cref{sec:justification}. To ensure this property, it is sufficient (\Cref{prop:justification}) to have the following constraints:
\begin{align}
    \quad \ub_1^{\top} \bDelta = 0, \quad \bDelta \vb_1 = 0.\label{eq:no-disturb}
\end{align} 
Equivalently, $\bDelta$ must lie in the orthogonal subspace:
\begin{align}
    \bDelta = \Pb_{\Ub^{\perp}} \bDelta \Pb_{\Vb^{\perp}}.\label{eq:ortho-subspace}
\end{align}
where $\Pb_{\Ub^{\perp}} = \Ib - \ub_1 \ub_1^{\top}$ is the projection orthogonal to $\ub_1$, and $\Pb_{\Vb^{\perp}} = \Ib - \vb_1 \vb_1^{\top}$ is the projection orthogonal to $\vb_1$. 

With this constraint, $\Wb - \eta\cdot S\cdot \bDelta$ has singular value $S$ along $\ub_1,\vb_1$, while other singular values are $\sigma_i(\Wb)$ perturbed by $\eta\cdot S\cdot \bDelta$. Since $\sigma_i(\Wb)\le \sigma_2(\Wb),\forall i\ge 2$, $\|\bDelta\|=1$ as well as $|\eta  S| \le S-\sigma_2(\Wb)$, therefore, it holds that $\sigma_1(\Wb - \eta\cdot S\cdot \bDelta) = S$, i.e., $\|\Wb - \eta\cdot S\cdot\bDelta\| = S$. 

Therefore, the optimization problem~\eqref{eq:P} is reduced to
\begin{equation}\label{eq:P-reduce} 
\begin{aligned}
    &\max_{\bDelta: \norm{\bDelta} \leq 1} \dotp{\Gb}{\bDelta} \\
    &\text{ subject to } \ub_1^\top\bDelta=0, \bDelta\vb_1=0.
\end{aligned}
\end{equation}
By projecting the gradient matrix $\Gb$ to the same orthogonal subspace in~\eqref{eq:ortho-subspace} to attain $\hat{\Gb}:= \Pb_{\Ub^{\perp}} \Gb \Pb_{\Vb^{\perp}}$,~\eqref{eq:P-reduce} becomes:
\begin{equation}\label{eq:P-reduce-compact} 
\begin{aligned}
    \max_{\bDelta: \norm{\bDelta} \leq 1} \dotp{\hat{\Gb}}{\bDelta}.
\end{aligned}
\end{equation}
And the optimal update is $\bDelta=\hat{\Ub}\hat{\Vb}^\top$, where $\hat{\Ub}\hat{\bSigma}\hat{\Vb}^\top=\text{SVD}(\hat{\Gb})$ is the SVD decomposition of the projected gradient matrix. 

The the analysis above yields \algname{} as shown in Algorithm~\ref{alg:tbd}, where $\mathrm{msign}(\Wb)$ for some matrix $\Wb$ with singular value decomposition of $\Ub,\bSigma,\Vb^\top=\text{SVD}(\Wb)$ is the matrix sign function defined as $\mathrm{msign}(\Wb):=\Ub\Vb^\top$~\citep{roberts1980linear,denman1976matrix}.
\begin{algorithm}[tb!]
\caption{\algname{}}
\begin{algorithmic}[1]
\setcounter{ALC@unique}{0}
\label{alg:tbd}
    \STATE \textbf{input:} $\Wb_0\in\mathbb{R}^{n_{\ell}\times n_{\ell-1}},\mu,\{\eta_t\}_{\geq1}$
    \STATE Set $\Wb_1\leftarrow\Wb_0$ (initialized following \cite[Parametrization~1]{yang2023spectral}), $\Mb_0\leftarrow\mathbf{0}$
    \FOR {$t=1,$ \textbf{to} $ T$}
        \STATE Compute gradient $\Gb_t=\nabla f_t(\Wb_t)$
        \STATE $\Mb_t=\mu\Mb_{t-1}+\Gb_t$
        \STATE Compute the top left singular vector $\ub_{1,t}$ and right singular vector $\vb_{1,t}$ of $\Wb_t$
        \STATE $\Pb_{\Ub_t^\perp}=\Ib-\ub_{1,t}\ub_{1,t}^\top$, $\Pb_{\Vb_t^\perp}=\Ib-\vb_{1,t}\vb_{1,t}^\top$
        \STATE $\bDelta_t=\mathrm{msign}(\Pb_{\Ub_t^\perp}\Mb_t\Pb_{\Vb_t^\perp})$
        \STATE $\Wb_{t+1}=\Wb_t-\eta_t\cdot S\cdot\bDelta_t$
    \ENDFOR
\end{algorithmic}
\end{algorithm}

For a sufficiently small effective learning rate $\eta$, the condition $\eta S\le S-\sigma_2(\Wb)$ can be satisfied with high probability. Therefore, \Cref{alg:tbd} is solid enough for training with moderately large width \emph{without} explicit spectral-norm normalization, e.g., $\mathtt{hidden\_size} \leq 4096$. However, for weight $\Wb$ with prohibitively large size, $S-\sigma_2(\Wb)$ may vanish to zero, which motivates some practical alleviations discussed below.

\subsection{Practical considerations for models with ultra-large width}

\begin{algorithm}[tb!]
\caption{\algname{} with Direct Rescaling}
\begin{algorithmic}[1]
\setcounter{ALC@unique}{0}
\label{alg:retraction}
    \STATE \textbf{input:} $\Wb_0\in\mathbb{R}^{n_{\ell}\times n_{\ell-1}},\mu,\{\eta_t\}_{t\geq1}$
    \STATE Set $\Wb_1\leftarrow\Wb_0$ (initialized following \cite[Parametrization~1]{yang2023spectral}), $\Mb_0\leftarrow\mathbf{0}$
    \FOR {$t=1,$ \textbf{to} $ T$}
        \STATE Compute gradient $\Gb_t=\nabla f_t(\Wb_t)$
        \STATE $\Mb_t=\mu\Mb_{t-1}+\Gb_t$
        \STATE Compute the top left singular vector $\ub_{1,t}$ and right singular vector $\vb_{1,t}$ of $\Wb_t$ \label{line:compute-top-singular-vectors}
        \STATE $\Pb_{\Ub_t^\perp}=\Ib-\ub_{1,t}\ub_{1,t}^\top$, $\Pb_{\Vb_t^\perp}=\Ib-\vb_{1,t}\vb_{1,t}^\top$
        \STATE $\bDelta_t=\mathrm{msign}(\Pb_{\Ub_t^\perp}\Mb_t\Pb_{\Vb_t^\perp})$
        \STATE $\Wb_{t+\frac{1}{2}}=\Wb_t-\eta_t\cdot S\cdot\bDelta_t$ \label{line:retraction:t+1/2}
        \STATE $\Wb_{t+1} \leftarrow S \cdot \Wb_{t+\frac{1}{2}} / \norm{\Wb_{t+\frac{1}{2}}}$ \label{line:ad-hoc-rescaling-alg2}
    \ENDFOR
\end{algorithmic}
\end{algorithm}

Given \Cref{fact:gap-vanish} that $\sigma_1 \approx \sigma_2$ for some ultra high-dimensional weight matrices under $\mup$ scaling, \Cref{alg:retraction} may serve as an approximate relief when $n_{\ell} \wedge n_{\ell-1} \gg 1$, leveraging a direct spectral rescaling in \Cref{line:ad-hoc-rescaling-alg2}. It is worth noting that \Cref{alg:retraction} is equivalent to \Cref{alg:tbd} if $\eta_t S\leq S - \sigma_2(\Wb_{t})$ or $\norm{\Wb_{t+1/2}} = S$ (in \Cref{line:retraction:t+1/2}), due to spectral cancelations between $\bDelta_{t}$ and $\Wb_{t}$. Even when it is the case, if $\norm{\Wb_{t+1/2}}$ dominates $\norm{\Wb_{t}}$ by a non-negligible margin, the direct rescaling in \Cref{alg:retraction} might make $\norm{\Wb_{t+1} - \Wb_{t}}$ smaller than the desired norm of $\sqrt{n_{\ell} / n_{\ell-1}}$ in the spectral condition of the neat update, which is not likely to hinder the training stability but may affect the perfectness of hyperparameter transfer. Therefore, \Cref{alg:retraction} can be served as an appropriate compromise between ideal principles and efficient implementations when the width is prohibitively large.

\begin{remark}
    Although \Cref{fact:gap-vanish} implies that $\sigma_1 - \sigma_2$ may be small for certain $\Wb$ with a stable rank of $\Omega(n_{\ell} \wedge n_{\ell-1})$ and large enough $n_{\ell} \wedge n_{\ell-1}$, it is worth noting that around the late stage of training (e.g., for $t > T/2$), (1) many relevant weights may have smaller stable ranks compared to their initialization (and thus $\sigma_1 - \sigma_2$ becomes larger with high probability), and (2) $\eta_t$ will be super small if following practically favorable learning rate schedulers~\citep{loshchilov2016sgdr,hu2024minicpm,bergsma2025straight}. Therefore, \Cref{line:ad-hoc-rescaling-alg2} in \Cref{alg:retraction} is not likely to take place too frequently after the very first portion of the training procedure.
\end{remark}

\begin{remark}
    The conceptually standard principle for solving constrained optimization problems \Cref{eq:P,eq:P-reduce} is strong duality, which we outline as another iterative alternative of \Cref{alg:tbd,alg:retraction} under a first-order oracle premise in \Cref{app:strong-duality}.
\end{remark}

\subsection{Should the spectral condition for weight matrices be time-independent?}\label{subsec:perspective-corr}

The $S$ in~\Cref{eq:P} is originally deemed to be time-independent because it is widely \textbf{believed} that the spectral condition $S = \sqrt{n_\ell / n_{\ell-1}}$ for $\Wb$ is crucial for preserving the norm of the feature vectors as well as their update scales~\citep{yang2021tensor,yang2023spectral}. However, for the \textbf{late}\footnote{See \Cref{sec:adamw+mup:cumulative} for an example of the late stage in LLM training (even with, e.g., Adam under $\mup$).} stage of training, e.g., when the number of training steps $T$ satisfies $T \gtrsim \mathrm{poly}(n_{\ell} \wedge n_{\ell-1})$,
\begin{enumerate}[label=\textbf{Issue \arabic*},align=left]
    \item the elements within $\Wb$ may become mutually dependent (due to the accumulation of the impact of data-dependent gradient information), especially when the training token budget is \emph{huge} yet $n_\ell$ and $n_{\ell-1}$ remain \emph{mild}; \label{point:corr}
    \item the activation vector $\hb^{(\ell-1)}$ and $\Wb^{(\ell)}$ may correlate in an intriguing yet complex way.\label{point:alignment-issue}
\end{enumerate}

\begin{remark}
\ref{point:alignment-issue} has been empirically studied in different contexts~\citep{everett2024scaling,haas2025surprising}. However, it is by far still hard to model it in an principled way that is easily implementable for practitioners to remedy this issue. We leave the data-dependent treatment of this issue as a future direction.
\end{remark}

Any of these two issues alone already breaks down the derivation of the spectral conditions~\citep[Condition~1]{yang2023spectral}.
To inspire future work on the refinement of spectral conditions, we initiate an effort \emph{on top of \Cref{alg:tbd}} towards a principled treatment of \ref{point:corr}. In detail, we introduce a minimalist correlation model for $\Wb$, which is promising to estimate \emph{on the fly} (i.e., during the normal training procedure) efficiently.

\begin{definition}[$\rho$-correlated weight]\label{def:corr}
    $\Wb \in \RR^{m \times n}$ is\footnote{We denote $m = n_{\ell}$ and $n = n_{\ell-1}$ in the presentation below to avoid notation clutter.} a $\rho$-correlated weight if $\{W_{i,j}\}_{(i,j) \in [m] \times [n]}\iidsim \cN(0, \sigma^2)$ and $\mathrm{corr}(W_{i,j}, W_{k,\ell}) = \rho$ for any $i\neq j$ or $k \neq \ell$.
\end{definition}
Given $\rho \geq 0$, $Z \sim \cN(0, 1)$, $\bPhi \in \RR^{m\times n}$ with $\iid$ standard Gaussian entries, and $\Jb \coloneqq \one_{m}\one_{n}^\top$, we have $\sigma^{-1}\Wb \stackrel{\mathrm{d}}{=} \sqrt{\rho}Z\Jb + \sqrt{1 - \rho} \bPhi$ if and only if $\Wb$ is a $\rho$-correlated weight. Thus, we identify 
\begin{align}
    \Wb \coloneqq \sigma\big( \sqrt{\rho}Z\Jb + \sqrt{1 - \rho} \bPhi \big) \label{eq:identification}
\end{align}
in the analysis below W.L.O.G. To further simplify the presentation, we make the following assumption, which is reasonable for many 2-dimensional weights in modern LLMs like \citet{bi2024deepseek}.
\begin{assumption}\label{assump:fixed-ratio}
    $m/n \equiv c$, where $c \in (0, 1]$ is a constant; i.e., the ratio between the input and output dimensions of $\Wb$ is fixed.
\end{assumption}
\Cref{assump:fixed-ratio} implies $\sqrt{m/n} = \sqrt{c} = \Theta(1)$; under which we make a further \emph{simplification} that, to preserve the elementwise magnitude of the activation vector after multiplied by $\Wb$, it is still desired that $\norm{\Wb} \asymp \sqrt{c}$ for any large $n$ in every iteration, i.e.,
\begin{desideratum}\label{desi:spectral-corr}
  For the $\rho$-correlated weight $\Wb$, $\norm{\Wb} = \Theta_{\PP}(1)$.
\end{desideratum}
\begin{remark}
    In the discussion above, we assume that $\rho \geq 0$, which is indeed a limitation. However, whenever this $\rho$-correlated weight is well-defined, i.e., the covariance matrix of the $mn$ random variables in $\Wb$ is positive semi-definite, we will always have $\rho \gtrsim -(mn)^{-1}$, as detailed in \Cref{sec:discuss-rho>=0}. It also implies that $\rho$ is not very likely to be negative for relatively large $n$ under this phenomenological model.
\end{remark}
 
To achieve \Cref{desi:spectral-corr}, it is vital to carefully track the relative order of $\sigma$ versus $\rho$ in \Cref{def:corr} compared to $n$. Therefore, it is natural to index them by $\sigma_n$ and $\rho_n$ for the $\Wb \in \RR^{m\times n}$ under \Cref{assump:fixed-ratio} in the following analysis.

We analyze $\norm{\Wb}_F$, $\norm{\Wb}$, and the \emph{stable rank} $\mathrm{srank}(\Wb) \coloneqq \norm{\Wb}_F^2 / \norm{\Wb}^2$ as follows, emphasizing their dependency on the scaling of $(\rho_n, \sigma_n)$. The regime of non-vanishing correlation is straightforward:
\begin{align}
    \rho_n \asymp 1 \Longrightarrow\quad \text{ \Cref{desi:spectral-corr} is satisfied if and only iff } \sigma_n \asymp n^{-1}, \label{eq:non-vanishing-corr}
\end{align}
whose proof deferred to \Cref{sec:non-vanishing-corr}. The regime of $\rho_n \ll 1$ is more interesting as shown below.
\begin{proposition}\label{prop:F-norm-vanishing-corr}
    Under \Cref{assump:fixed-ratio}, if $\rho_n \ll 1$, $ { \norm{\Wb}_{F} } / \big({ \sigma_n \sqrt{mn} }\big) \xrightarrow[n\to\infty]{\text{a.s.}} 1$.
\end{proposition}
\Cref{prop:F-norm-vanishing-corr} manifests that $\norm{ \Wb }_F$ behaves consistently in the regime of vanishing correlation, yet $\norm{\Wb}$ behaves distinctly in three different sub-regimes: $\rho_n \ll n^{-1}$, $\rho_n \gg n^{-1}$, and $\rho_n \asymp n^{-1}$.
\begin{proposition}\label{prop:spectral-norm-vanishing-corr}
    Under the same assumptions in \Cref{prop:F-norm-vanishing-corr} and the identification \Cref{eq:identification}:
\begin{align}
    \rho_n \ll n^{-1} &\Longrightarrow\quad \frac{\norm{\Wb}}{\sigma_n(\sqrt{m} + \sqrt{n})} \xrightarrow[n\to\infty]{\text{a.s.}} 1, \label{eq:spectral-ll}\\
    \rho_n \gg n^{-1} &\Longrightarrow\quad \frac{\norm{\Wb}}{\sigma_n\sqrt{mn\rho_n}} \xrightarrow[n\to\infty]{\text{a.s.}} |Z| = \Theta_\PP(1), \label{eq:spectral-gg}\\
    \rho_n \asymp n^{-1} &\Longrightarrow\quad \frac{\norm{\Wb}}{\sigma_n(\sqrt{m} + \sqrt{n})} = O_\PP(1), \label{eq:spectral-asymp-Op}
\end{align}
and in particular, when $\lim_{n\to\infty} n \rho_n = \tau \in (0, +\infty)$,
\begin{align}
    \frac{\norm{\Wb}}{ \sigma_n(\sqrt{m} + \sqrt{n})} \xrightarrow[n\to\infty]{\text{a.s.}} \ind{\{Z^2\tau\sqrt{c} \leq 1\}} + \ind{\{Z^2\tau\sqrt{c} > 1\}} \cdot \frac{\sqrt{(Z^2\tau + 1)(Z^2\tau c + 1)}}{|Z|\big( 1 + \sqrt{c} \big) \sqrt{\tau}} = \Theta_\PP(1). \label{eq:spectral-asymp-exact}
\end{align}
\end{proposition}
\Cref{subsec:proof-F-norm-vanishing-corr} (resp. \Cref{subsec:proof-spectral-norm-vanishing-corr}) details the proof of \Cref{prop:F-norm-vanishing-corr} (resp. \Cref{prop:spectral-norm-vanishing-corr}).
The main take-away of \Cref{prop:F-norm-vanishing-corr,prop:spectral-norm-vanishing-corr,eq:non-vanishing-corr} lies in the regime $\rho_n \gg n^{-1}$, in which we need 
\begin{align}
    \textcolor{red}{\sigma_n \asymp n^{-1}\rho_n^{-1/2}} \label{eq:key-diff}
\end{align}
to achieve \Cref{desi:spectral-corr}; and $\mathrm{srank}(\Wb) = \Theta_\PP(\rho_n^{-1})$ in this case, which is smaller than the common $\mathrm{srank} \sim mn / (\sqrt{m} + \sqrt{n})^2$ in the $\mup$ literature. In contrast, in regimes of $\rho_n \lesssim n^{-1}$, \Cref{prop:spectral-norm-vanishing-corr} indicates that the commonly used $\sigma_n \asymp n^{-1/2}$ is necessary and sufficient for \Cref{desi:spectral-corr}. Therefore, as long as we can efficiently estimate $\rho_n$ or its scaling exponent w.r.t. $n$ on the fly, we may be able to {adapt} $\Wb$ to its current \emph{correlation status}.

\subsubsection{Correlation estimation and its potential utilization}

Empirically, given the weight matrix from a certain checkpoint, we can estimate $\rho_n$ by the \emph{method of moments} in $O(mn)$ time as
\begin{align}
    \hat{\rho}_{n}^{\mathrm{MoM}} \leftarrow \frac{mn \overline{W}^2 - \hat{\sigma}_n^2}{(mn-1)\hat{\sigma}_n^2}, \label{eq:MoM}
\end{align}
where $\overline{W} \leftarrow (mn)^{-1}\sum_{i,j} W_{i,j}$ and $\hat{\sigma}_n^2 \leftarrow (mn)^{-1}\sum_{i,j} W_{i,j}^2$.

Under the belief of $\rho_n = 1 / \poly(n)$, it is possible to fit a linear regression for $\log \hat{\rho}_n$ v.s. $\log n$ across different model widths to determine the scaling exponent of $\rho_n$. Moreover, our analysis \Cref{eq:key-diff} above suggests that it is viable to introduce a threshold $C$ such that when the estimator $\hat{\rho}_{n}^{\mathrm{MoM}, (t)}$ \Cref{eq:MoM} for the current step $t$ exceeds $C \cdot (n^{-1/2} + m^{-1/2})$ for the first time, we scale $\Wb_{t}$ by a factor of $\sqrt{\hat{\rho}_{n}^{\mathrm{MoM}, (t-1)} / \hat{\rho}_{n}^{\mathrm{MoM}, (t)}}$.

\begin{remark}
    The correlation utilization approach outlined above can be more computationally efficient than other iterative approaches (e.g., \Cref{line:ad-hoc-rescaling-alg2} of \Cref{alg:retraction}) for adjusting the spectral status of $\Wb_{t+1}$. In particular, we recall that \Cref{line:ad-hoc-rescaling-alg2} of \Cref{alg:retraction} requires an additional invocation of power (or Lanczos) iteration because it needs the concrete spectral norm computation.
\end{remark}

\section{Discussion}

Our projection-based \algname{} algorithm enforces the required spectral properties of weight matrices throughout the training process in an on-the-fly manner, without imposing overly rigid structural constraints (e.g., orthogonal rows or columns~\citep{modula-docs}). Instead, it allows learned representations to be encoded through correlated features, while still satisfying \Cref{desi:spectral-corr} at each optimization step, even in the late stages of training.

\subsection{Relevance and limitations}

\paragraph{Originality and Relevance.} Beyond the introduction of new algorithms, our work draws on well-established tools from the theory of Gaussian random matrices~\citep{benaych2012singular,speicher2020lecture} and classical point estimation methods, including the method of moments. While these techniques are independently well studied, our contribution lies in their principled integration into the study of spectral conditions under $\mup$, leveraging training-time information across the entire optimization process to enable practical and reliable matrix-based optimizers. 

\paragraph{Limitations.} For \algname{}, we still need explicit spectral normalization for super large-width cases in \Cref{alg:retraction}. Also, the modeling framework in \Cref{subsec:perspective-corr} for alleviating \Cref{point:corr} is far from comprehensive, and a neat and sound mitigation of \Cref{point:alignment-issue} is left as a future work.

\subsection{On the impact of weight decay v.s. explicit constraints}
It is worth noticing that our norm constraint \Cref{eq:P} never rules out the necessity of weight decay (which is not discussed in the current manuscript), namely, weight decay is not necessarily redundant even under explicit weight magnitude normalization: projected gradient methods alone will converge to stationary points if certain stationary points satisfy the norm constraints, but decoupled weight decay introduces a ``shrinkage factor'' $(1 - \eta_t \lambda)$, which may pull the model away from any stationary points during pre-training even when the momentum is already near-vanishing.


\appendix

\section{The Admissible Range of $\eta$}

In this section, we demonstrate that, the next step $t+1$ of \Cref{alg:tbd} automatically satisfies the constraint in \Cref{eq:P} if and only if $\eta_t \cdot S \leq \sigma_1(\Wb_{t}) - \sigma_2(\Wb_t)$ for the current step $t$.

\subsection{$\eta \cdot S < \sigma_{1}(\Wb) \coloneqq S$ is not enough: $\eta \cdot S \le \sigma_1(\Wb) - \sigma_2(\Wb)$ is necessary}

\begin{align*}
 \forall \delta > 0, \eta \coloneqq 0.8 + \delta,  \Wb \coloneqq \begin{bmatrix}
        1 & 0 \\
        0 & 0.2
    \end{bmatrix}, \bDelta \coloneqq \begin{bmatrix}
        0 & 0 \\
        0 & -1 
    \end{bmatrix}; 
    \Ub \coloneqq \Vb \coloneqq \begin{bmatrix}
        1 & 0 \\
        0 & 1
    \end{bmatrix}, \bSigma \coloneqq \Wb.
\end{align*}
Then the singular value decomposition of $\Wb$ is $\Wb = \Ub\Sigma \Vb^\top$, whose $\sigma_1(\Wb) = 1$, $\sigma_2(\Wb) = 0.2$, and $\sigma_1(\Wb) - \sigma_{2}(\Wb) = 0.8$, which means $\eta S < \sigma_1(\Wb)$ but $\eta S > \sigma_1(\Wb) - \sigma_2(\Wb)$; and most importantly:
\begin{align}
    \Wb - \eta \cdot S \cdot\bDelta = \begin{bmatrix}
        1 & 0 \\
        0 & 1 + \delta
    \end{bmatrix},
\end{align}
whose maximal singular value is $1 + \delta > \sigma_{1}(\Wb) \coloneqq S = 1$.

\subsection{$\eta\cdot S\le \sigma_1(\Wb)-\sigma_2(\Wb)$ is sufficient}\label{sec:justification}

\begin{proposition}\label{prop:justification}
Let $\Wb \in \RR^{n_{\ell} \times n_{\ell-1}}$ have its SVD $\Ub \bSigma \Vb^\top$, where $\Ub = [\ub_1, ...]$, $\Vb = [\vb_1, ....]$, and $\bSigma = \mathrm{diag}(\sigma_1, \sigma_2, ...)$ with $S \coloneqq \sigma_1 > \sigma_2$. Suppose $\norm{ \bDelta } = 1$, $\ub_1^\top \bDelta = \zero_{n_{\ell-1}}$, $ \bDelta \vb_1 = \zero_{n_{\ell}}$ and $|\eta S| \leq  \sigma_1 - \sigma_2$, then $\norm{ \Wb - \eta \cdot S \cdot\bDelta } = S$.
\end{proposition}
\begin{proof}
 $\forall \xb \in \RR^{n_{\ell-1}}$ with $\norm{\xb}_2 = 1$, we decompose it into $\xb = \alpha \vb_1 + \yb$,
where $\yb \in \mathrm{span}\big( \{\vb_i\}_{i > 1} \big)$ and $\norm{\yb}_2^2 = 1 - \alpha^2$. Note that 
\begin{align*}
    \norm{(\Wb - \eta \cdot S \cdot \bDelta)\xb}_2^2  &= \norm{ \alpha \sigma_1 \ub_1 + (\Wb - \eta \cdot S \cdot \bDelta)\yb }_2^2 \\
    &= \alpha^2\sigma_1^2 + \norm{(\Wb - \eta \cdot S \cdot \bDelta)\yb}_2^2 \\
    &\leq \alpha^2\sigma_1^2 + (\norm{\Wb\yb}_2 + |\eta S|\norm{\bDelta\yb}_2 )^2 \\
    &\leq \alpha^2\sigma_1^2 + (\sigma_2 + |\eta S|)^2 \norm{\yb}_2^2 \\
    &\leq \alpha^2\sigma_1^2 + \sigma_1^2  (1-\alpha^2) = \sigma_1^2,
\end{align*}
where the first equality is due to $\bDelta \vb_1 = \zero$, the second equality follows from $(\Wb \yb) \in \mathrm{span}\big( \{\ub_i\}_{i > 1} \big)$ and $\ub_1^\top \bDelta = \zero \Rightarrow (\bDelta \yb) \perp \ub_1$; the first inequality is effectively the triangle inequality, the second inequality is by $\yb \perp \vb_1$ and $\norm{\bDelta} = 1$, and the last inequality follows from $|\eta S| \leq \sigma_1 - \sigma_2$ and $\norm{\yb}=  \sqrt{1 - \alpha^2}$.
\end{proof}

\subsection{$\sigma_1(\Wb) - \sigma_2(\Wb)$ vanishes as width goes large}

\begin{fact}[{See, e.g., \citealt[Introduction]{bloemendal2013limits}}]\label{fact:gap-vanish}
    Let $\Ab = (A_{ij})_{(i,j) \in [m] \times [n]}$ be a random matrix with $\iid$ real entries such that $\EE A_{11} = 0$, $\EE A_{11}^2 = 1$, and $\EE A_{11}^4 < \infty$; suppose $m/n \equiv c$ as $n \to \infty$, i.e., $m/n$ is always a constant $c$ that does not vary with $n$; let $\Wb \coloneqq n^{-1/2}\Ab$ and $\sigma_{k} \coloneqq \sigma_k(\Wb)$ be the $k$-th singular value of $\Wb$, then as $n \to \infty$, $(\sigma_1 - \sigma_2) \stackrel{\mathrm{a.s.}}{\longrightarrow} 0$.
\end{fact}

\section{A Conceptual Alternative}\label{app:strong-duality}

\Cref{alg:tbd,alg:retraction} both require the computation of the top singular vectors, whose cost is tolerable if implemented via approximation algorithms in numerical linear algebra, such as \emph{power iteration} and \emph{Lanczos iteration}~\citep{trefethen2022numerical}. Moreover, it is possible to approximately solve \Cref{eq:P} based on strong duality and solve for it iteratively based on a first-order oracle iteratively as follows.

We still work with $\sigma_1 > \sigma_2$ and {relax} \Cref{eq:P-reduce} to
\begin{equation}\label{eq:P-relaxed} 
\begin{aligned}
    &\max_{\bDelta: \norm{\bDelta} \leq 1} \dotp{\Gb}{\bDelta} \\
    &\text{ subject to } \ub_1^\top\bDelta\vb_1=0.
\end{aligned}
\end{equation}
By strong duality~\citep{beck2017first}, we equivalently formulate the dual problem of \Cref{eq:P-relaxed} as
\begin{align}
    \min_{\nu \in \RR} \max_{\bDelta: \norm{\bDelta} = 1}   \dotp{\Gb + \nu \ub_1\vb_1^\top}{\bDelta }, \label{eq:D}
\end{align}
the inner problem of which is well-known (See, e.g., \citealt{bernstein2024old}) to have
\begin{align}
    \bDelta(\nu) \coloneqq \mathrm{msign}(\Gb + \nu \ub_1\vb_1^\top) \label{eq:def-of-Delta-nu}
\end{align}
as a maximizer and $\norm{\Gb + \nu \ub_1\vb_1^\top}_{*}$ as its maximum. Consequently, it suffices to obtain $\nu_{\min} \in \argmin_{\nu} \norm{\Gb + \nu \ub_1\vb_1^\top}_{*}$ and return $\bDelta(\nu_{\min})$ as a solution. Recall that $F(\nu) \coloneqq \norm{\Gb + \nu \ub_1\vb_1^\top}_{*}$ is convex in $\nu$~\citep{beck2017first}, whose subdifferential is obtained by the chain rule~\citep[Theorem~3.43]{beck2017first} and \Cref{fact:nuclear-subdiff} as
\begin{align*}
    \partial F(\nu) = \{ \dotp{\ub_1\vb_1^\top}{\Ab\Bb^\top + \Cb}: \Ab^\top \Cb = \zero, \Cb \Bb = \zero, \norm{\Cb} \leq 1 \},
\end{align*}
where $\Ab\bGamma\Bb^\top$ is the SVD of $\Gb + \nu \ub_1\vb_1^\top$. In particular, one subgradient at $\nu$ could be
\begin{align}
    \dotp{\ub_1\vb_1^\top}{\bDelta(\nu)} \in \partial F(\nu). \label{eq:subgradient-oracle}
\end{align}
\Cref{eq:def-of-Delta-nu,eq:subgradient-oracle} entail that we can build in practice an approximate first-order oracle for $F(\nu)$ via the standard approximation routine (Newton-Schulz, see, e.g., \citealt{bernstein2024old}) of $\mathrm{msign}$. The rest is to invoke subgradient descent to solve for $\nu_{\min} \in \argmin_{\nu \in \RR} F(\nu)$.

\begin{remark}
    \Cref{eq:P-relaxed} is another proxy primal problem of \Cref{eq:P}, which can also be interpreted through the aspect of subgradient approximation. In particular, the convexity of $\eta S \mapsto \norm{\Wb - \eta \cdot S \cdot\bDelta}$ implies $\norm{\Wb - \eta \cdot S \cdot \bDelta} \geq \norm{\Wb} - \eta S\dotp{\ub_1\vb_1^\top}{\bDelta}$, where we utilize $(\ub_1\vb_1^\top) \in \partial \norm{\Wb}$ following \Cref{fact:spectral-subdiff}. The induction hypothesis $\norm{\Wb} = S$ and the constraint $\norm{\Wb - \eta \cdot S \cdot \bDelta} = S$ then indicate $\ub_1^\top\bDelta \vb_1 \geq 0$; and thus \Cref{eq:P-relaxed} can be viewed as an approximate relaxation of \Cref{eq:P} under the belief that the subgradient approximation is nearly an equality (which is convincing when, e.g., $\eta$ is relatively small).
\end{remark}

\section{Auxiliary Lemmas}

We recall two folklore on subdifferential, both of which are elaborated in, e.g., \citet{watson1992characterization}.

\begin{fact}\label{fact:nuclear-subdiff}
The subdifferential of the nuclear norm $\norm{\cdot}_*$ evaluated at $\Mb$ is
\begin{align*}
    \partial \norm{\Mb}_{*} = \{ \Ub \Vb^\top + \Wb : \Ub^\top \Wb = \zero, \Wb\Vb = \zero, \norm{\Wb} \leq 1\}
\end{align*}
if the SVD of $\Mb$ is $\Ub \bSigma \Vb^\top$. In particular, $\mathrm{msign}(\Mb) \in \partial \norm{\Mb}_{*}$.
\end{fact}

\begin{fact}\label{fact:spectral-subdiff}
The subdifferential of the spectral norm $\norm{\cdot}$ evaluated at $\Mb$ is
\begin{align*}
    \partial  \norm{\Mb} = \{ \Qb: \norm{\Mb} = \dotp{\Mb}{\Qb}, \norm{\Qb}_{*} \leq 1 \}.
\end{align*}
\end{fact}

\section{Detailed Discussions on the Correlation Model}\label{sec:corr-discussion}

\subsection{On the non-negativeness of $\rho$}\label{sec:discuss-rho>=0}

We omit the subscript (in $n$) for $\sigma$ and $\rho$ in this subsection and let $K \coloneqq mn$ to avoid notational clutter. In \Cref{def:corr}, the covariance matrix $\bSigma = \sigma^2 \big( (1-\rho)\Ib_{K} + \rho \one_{K}\one_{K}^\top)\big)$ for all random variables in $\Wb$ must satisfy $\bSigma \succeq \zero_{K \times K}$. Note that the eigenvalue of $\bSigma$ corresponding to all vectors orthogonal to $\one$ is $\sigma^2(1-\rho) \geq 0$, and the eigenvalue corresponding to $\one$ is $\sigma^2\big( (1-\rho) + \rho K\big)$; which implies that $\bSigma \succeq \zero \Leftrightarrow (1-\rho) + \rho K \geq 0 \Leftrightarrow \rho \geq -(K-1)^{-1}$. Therefore, \Cref{def:corr} implicitly constrains $\rho$ to be nearly non-negative for sufficiently large $n$.

\subsection{Proof of \Cref{eq:non-vanishing-corr}}\label{sec:non-vanishing-corr}

In this case of non-vanishing correlation, i.e., $\rho_n \asymp 1$, the pairwise weight correlation is bounded (from below) away from zero for any sufficiently large $n$, which is a ``super''-correlated regime. Recall the identification in \Cref{subsec:perspective-corr} of $\sigma_n^{-1} \Wb \coloneqq  \sqrt{\rho_n}Z\Jb + \sqrt{1 - \rho_n} \bPhi$, under which
\begin{align*}
 \sigma_n \sqrt{\rho_n} |Z| \norm{\Jb} - \sigma_n \sqrt{1 - \rho_n} \norm{\bPhi}  \leq  \norm{\Wb} \leq \sigma_n \sqrt{\rho_n} |Z| \norm{\Jb} + \sigma_n \sqrt{1 - \rho_n} \norm{\bPhi}.
\end{align*}
Under \Cref{assump:fixed-ratio}, $\norm{\Jb} = \Theta(n)$, and $n^{-1/2}{\norm{\bPhi}} \in \Theta_\PP(1) \ni |Z|$ (again due to the classic Bai-Yin theorem, see, e.g., \citealt{speicher2020lecture}); since $0 < \liminf_{n\to\infty} \rho_n \leq \limsup_{n\to\infty} \rho_n \leq 1$, we conclude that:  \Cref{desi:spectral-corr} is satisfied if and only if $\sigma_n \asymp n^{-1}$.

\subsection{Proof of \Cref{prop:F-norm-vanishing-corr}}\label{subsec:proof-F-norm-vanishing-corr}
We put all randomness in one probability space to establish pointwise convergence. In detail, let $\{\phi_{ij}\}_{i,j\geq1}$ be an (countably) infinite two-dimensional array of $\iid$ random variables following $\cN(0, 1)$. Explicitly write $m_n \coloneqq cn$ and $N_n \coloneqq m_n \cdot n$ to emphasize the dependency on $n$, and denote the submatrix $\{\phi_{ij}\}_{1\leq i \leq m_n, 1 \leq j \leq n}$ by $\bPhi^{(n)}$, and denote $\one_{m_n}\one_{n}^\top$ by $\Jb^{(n)}$. Our identification \Cref{eq:identification} thus reads $\sigma_n^{-1}\Wb^{(n)} = \sqrt{\rho_n}Z \Jb^{(n)} + \sqrt{1 - \rho_n}\bPhi^{(n)}$, for which
\begin{align}
    \frac{\norm{\Wb^{(n)}}_F^2}{N_n\sigma_n^2} &= \rho_nZ^2 + 2Z\sqrt{\rho_n(1-\rho_n)} \cdot \frac{\bigdotp{ \Jb^{(n)} }{ \bPhi^{(n)} }}{N_n} + (1 - \rho_n) \frac{\norm{\bPhi^{(n)}}_F^2}{N_n}  \notag\\
    &=  \underbrace{\rho_nZ^2}_{\text{(i)}} + 2Z\sqrt{\rho_n(1-\rho_n)} \cdot \underbrace{\frac{\bigdotp{ \Jb^{(n)} }{ \bPhi^{(n)} }}{N_n}}_{\text{(ii)}} + (1 - \rho_n) \underbrace{N_n^{-1} \sum_{i=1}^{m_n}\sum_{j=1}^{n} \phi_{ij}^2}_{\text{(iii)}}. \notag
\end{align}
Recall that $\rho_n \ll 1$ means $\rho_n \stackrel{n\to\infty}{\longrightarrow} 0$, and $Z \sim \cN(0, 1) \in \Theta_\PP(1)$. Also, the Strong Law of Large Numbers gives $\text{(ii)} \stackrel{\mathrm{a.s.}}{\longrightarrow} 0$ and $\text{(iii)} \stackrel{\mathrm{a.s.}}{\longrightarrow} 1$. Therefore, the continuous mapping theorem asserts that $\norm{\Wb^{(n)}}_F / \big( \sigma_n\sqrt{N_n} \big) \stackrel{\mathrm{a.s.}}{\longrightarrow} 1$.

\subsection{Proof of \Cref{prop:spectral-norm-vanishing-corr}}\label{subsec:proof-spectral-norm-vanishing-corr}
Our setup of the probability space and the premise of $\rho_n \ll 1$ are the same as those in \Cref{subsec:proof-F-norm-vanishing-corr}.

\paragraph{Case of $\rho_n \ll n^{-1}$.}
\begin{align*}
    \frac{\Wb}{\sigma_n\big( \sqrt{m} + \sqrt{n} \big)} &= Z \underbrace{\frac{\sqrt{\rho_n}}{\sqrt{m} + \sqrt{n}} \one_{m}\one_{n}^\top}_{\Ab_n} + \underbrace{\sqrt{1 - \rho_n}}_{\to 1} \cdot \underbrace{\frac{\bPhi}{\sqrt{m} + \sqrt{n}}}_{\Bb_n}.
\end{align*}
$\norm{\one_{m}\one_{n}^\top} = \sqrt{mn}$ and $\rho_n \ll n^{-1}$ imply $\norm{\Ab_n} \ll 1$. The Bai-Yin theorem gives $\norm{\Bb_n} \stackrel{\mathrm{a.s.}}{\longrightarrow} 1$. Also, the triangle inequality asserts
\begin{align*}
    \sqrt{1 - \rho_n}\norm{\Bb_n} - |Z|\norm{\Ab_n} \leq \mathrm{LHS} \leq \sqrt{1 - \rho_n}\norm{\Bb_n} + |Z|\norm{\Ab_n},
\end{align*}
for which we pass through the limit to obtain \Cref{eq:spectral-ll} by noticing that $\PP(|Z| < \infty) = 1$.

\paragraph{Case of $\rho_n \gg n^{-1}$.}
\begin{align*}
    \frac{{\Wb}}{\sigma_n\sqrt{mn\rho_n}} &= Z \cdot \frac{\one_{m}\one_{n}^\top}{\sqrt{mn}} + \underbrace{\sqrt{\frac{1 - \rho_n}{\rho_n}}}_{\ll \sqrt{n}} \frac{\bPhi}{\sqrt{mn}} = Z \cdot \frac{\one_{m}\one_{n}^\top}{\sqrt{mn}} + \frac{\bPhi}{\sqrt{m}} \cdot o(1),
\end{align*}
where $\norm{ \one_{m}\one_{n}^\top / \sqrt{mn} } = 1$, and the Bai-Yin theorem gives $m^{-1/2}{\norm{\bPhi}} \stackrel{\mathrm{a.s.}}{\longrightarrow} 1 + c^{-1/2}$. Thus, the triangle inequality yields
\begin{align*}
  |Z| - m^{-1/2}{\norm{\bPhi}} \cdot o(1)  \leq \mathrm{LHS} \leq |Z| + m^{-1/2}{\norm{\bPhi}} \cdot o(1),
\end{align*}
for which we pass through the limit to obtain \Cref{eq:spectral-gg}.

\paragraph{Case of $\rho_n \asymp n^{-1}$.} \Cref{eq:spectral-asymp-Op} again follows from a standard triangle-inequality argument with $\norm{\one_{m}\one_{n}^\top} = \sqrt{mn}$ and the Bai-Yin theorem. For the exact result in \Cref{eq:spectral-asymp-exact}, note that
\begin{align*}
    \frac{\Wb}{\sigma_n\big( \sqrt{m} + \sqrt{n} \big)} &= \sqrt{\rho_n}Z \frac{\one_m\one_n^\top}{\sqrt{m} + \sqrt{n}} + \sqrt{1 - \rho_n} \frac{\bPhi}{\sqrt{m} + \sqrt{n}} \\
    &= \sqrt{1  -\rho_n} \frac{\bPhi}{\sqrt{m} + \sqrt{n}} + \frac{Z\sqrt{n\rho_n}}{(1+c^{1/2})} c^{1/2} \big( m^{-1/2}\one_m \big) \big(n^{-1/2}\one_n^\top\big) \\
    &= \underbrace{\sqrt{1  -\rho_n} \frac{\bPhi}{\sqrt{m} + \sqrt{n}}}_{\Xb_n} + \frac{Z\sqrt{\tau}}{(1+c^{-1/2})} (mn)^{-1/2}\one_m\one_n^\top + \underbrace{\frac{Z\big( \sqrt{n\rho_n}-\sqrt{\tau} \big)}{(1+c^{-1/2})}  (mn)^{-1/2}\one_m\one_n^\top }_{\Cb_n}.
\end{align*}
Since we additionally assume $n\rho_n \to \tau \in (0, +\infty)$, $\norm{\Cb_n} \stackrel{\mathrm{a.s.}}{=} o(1)$. The quarter-circle law (also known as the Marchenko–Pastur Law)~\citep{bai2010spectral} asserts that the spectral distribution of the singular values of $\Xb_n$ converges weakly to $\mu(\ud s)$ given by
\begin{align*}
    \mu(\ud s) = \frac{\sqrt{4c - \big( (1+\sqrt{c})^2s^2 - 1 - c \big)^2}}{\pi c s} \ind{ \Big\{ s \in \Big(\frac{1-\sqrt{c}}{1+\sqrt{c}}, 1 \Big) \Big\} } \ud s.
\end{align*}
Also note that $Z$ is independent of $\bPhi$ and $\one_m\one_n^\top$ serves as a rank-one perturbation, a direct application of \citet[Theorem~2.8]{benaych2012singular}, in particular, \citet[Section~3.1]{benaych2012singular}, yields
\begin{align*}
    \mathrm{LHS} \stackrel{\mathrm{a.s.}}{\longrightarrow} \begin{cases}
        1, & |Z|c^{1/4}\tau \leq 1, \\
        \frac{\sqrt{(Z^2\tau + 1)(Z^2\tau c + 1)}}{|Z|\big( 1 + \sqrt{c} \big) \sqrt{\tau}}, & |Z|c^{1/4}\tau > 1;
    \end{cases}
\end{align*}
which is equivalent to \Cref{eq:spectral-asymp-exact}.

\section{On the Dominating Token Budget Threshold for Adam under $\mup$}\label{sec:adamw+mup:cumulative}

Recall that \href{https://huggingface.co/docs/transformers/main/model_doc/nanochat#transformers.NanoChatConfig.initializer_range}{\texttt{initializer\_range}} controls the standard deviation of weight initialization (which can be, e.g., $0.02$ in certain open-source libraries) of the proxy model for AdamW under $\mup$, and $n$ is the width, i.e., the input dimension of an arbitrary two-dimensional weight $\Wb$. For Adam under $\mup$, if we ignore weight decay and approximate Adam using the vanilla SignSGD, and assume that $\{\bDelta_t\}_{t=1}^{T}$ are roughly in the same direction (or from similar distributions). Suppose $\eta / 2 \approx T^{-1}\sum_{t=1}^{T} \eta_t$ (like cosine \texttt{lr} scheduler), where $\eta$ is the peak \texttt{lr}. Under these simplifications, the cumulative neat update will dominate the initialization if the total number of updates, $T$, satisfies
\begin{align}
    \frac{\mathtt{initializer\_range}}{\sqrt{n}} \leq 0.5\eta \frac{\mathtt{base\_width}}{n} \cdot T \Longleftrightarrow T \geq 2\eta^{-1}  \sqrt{n} \cdot \frac{\mathtt{initializer\_range}}{\mathtt{base\_width}}.
\end{align}
Therefore, the dominating token budget threshold (ignoring gradient accumulation, etc.) is roughly $\mathtt{batch\_size} \cdot 2\eta^{-1}  \sqrt{n} \cdot \frac{\mathtt{initializer\_range}}{\mathtt{base\_width}}$ for Adam under $\mup$. The heuristics in this section ignores the impact of weight decay, which is itself a tricky factor in practice~\citep{fan2025robust,kosson2025weight}; and is thus left as another interesting future direction.

\bibliographystyle{ims}
\bibliography{ref}

\end{document}